\documentclass{article}

\usepackage{microtype}
\usepackage{graphicx}
\usepackage{subfigure}
\usepackage{booktabs} % for professional tables

\usepackage{hyperref}

\usepackage[preprint]{icml2021}

\usepackage{tikz}
\usepackage{amssymb}
\usepackage{amsmath}
\usepackage{amsthm}
\usepackage{mathtools}
\usepackage{enumitem}

\newcommand{\compilehidecomments}{false}%HIDE comments
\ifthenelse{ \equal{\compilehidecomments}{true} }{%
	\newcommand{\huazheng}[1]{}
}{
\newcommand{\huazheng}[1]{{\color{blue!50!black}  [\text{Huazheng:} #1]}}
}

\ifthenelse{ \equal{\compilehidecomments}{true} }{%
	\newcommand{\zhiyuan}[1]{}
}{
\newcommand{\zhiyuan}[1]{{\color{green!50!black}  [\text{zhiyuan:} #1]}}
}

\ifthenelse{ \equal{\compilehidecomments}{true} }{%
	\newcommand{\chuanhao}[1]{}
}{
\newcommand{\chuanhao}[1]{{\color{yellow!50!black}  [\text{chuanhao:} #1]}}
}

\newcommand{\mt}{\mathsf{T}}
\DeclareMathOperator*{\argmax}{arg\,max}

\def \bbE {\mathbf{E}}

\def \btheta {\bm \theta}

\def \btheta {\mathrm{\boldsymbol{\theta}}}

\def \bA {\mathbf{A}}
\def \bx {\mathbf{x}}
\def \bv {\mathbf{v}}

\def \cA {\mathcal{A}}

\def \bb {\mathbf{b}}
\def \bI {\mathbf{I}}

\def \bE {\mathbf{E}}

\newtheorem{assumption}{Assumption}

\newtheorem{theorem}{Theorem}
\newtheorem{corollary}{Corollary}
\newtheorem{lemma}{Lemma}

\begin{document}

\twocolumn[
\icmltitle{Incentivizing Exploration in Linear Bandits under Information Gap}

% \icmltitle{Submission and Formatting Instructions for \\
%           International Conference on Machine Learning (ICML 2020)}



\icmlsetsymbol{equal}{*}

\begin{icmlauthorlist}
\icmlauthor{Huazheng Wang}{1}
\icmlauthor{Haifeng Xu}{1}
\icmlauthor{Chuanhao Li}{1}
\icmlauthor{Zhiyuan Liu}{2}
\icmlauthor{Hongning Wang}{1}
\end{icmlauthorlist}

\icmlaffiliation{1}{University of Virginia}
\icmlaffiliation{2}{University of Colorado, Boulder}
% \icmlaffiliation{to}{Department of Computation, University of Torontoland, Torontoland, Canada}
% \icmlaffiliation{goo}{Googol ShallowMind, New London, Michigan, USA}
% \icmlaffiliation{ed}{School of Computation, University of Edenborrow, Edenborrow, United Kingdom}

\icmlcorrespondingauthor{Huazheng Wang}{hw7ww@virginia.edu}
% \icmlcorrespondingauthor{Cieua Vvvvv}{c.vvvvv@googol.com}
% \icmlcorrespondingauthor{Eee Pppp}{ep@eden.co.uk}

\icmlkeywords{Linear Bandits, Incentivizing Exploration}

\vskip 0.3in
]



\printAffiliationsAndNotice{}  % leave blank if no need to mention equal contribution

\begin{abstract}

We study the problem of incentivizing exploration for myopic users in linear bandits, where the users tend to exploit arm with the highest predicted reward instead of exploring. In order to maximize the long-term reward, the system offers compensation to incentivize the users to pull the exploratory arms, with the goal of balancing the trade-off among exploitation, exploration and compensation.
We consider a new and practically motivated setting  where the context features observed by the user are more \emph{informative} than those used by the system, e.g., features based on users' private information are not accessible by the system. We propose a new method to incentivize exploration under such information gap, and prove that the method achieves both sublinear regret and sublinear compensation.  We theoretical and empirically analyze the added compensation due to the information gap, compared with the case that the system has access to the same context features as the user, i.e., without information gap. We also provide a compensation lower bound of our problem.
%Add
\end{abstract}

%\zhiyuan{Compare the result with \cite{kannan2017fairness}. They provide the regret and compensation bound without information gap, see Theorem 6.1 and Corollary 6.2, 6.3.} 

\section{Introduction}

%The classical Multi-armed Bandits (MAB)  is a sequential decision-making game, where the system iteratively pulls an arm, observes the reward of that action from an unknown distribution. An important variant of MAB is the linear bandits~\cite{Auer02,li2010contextual,Improved_Algorithm}, where expected reward of an arm is an unknown linear function of the context features associated with the arm. The goal of the system is to maximize the total rewards during the game. The system needs to carefully balance the decision between \emph{exploitation} by pulling the arm with current best estimated reward,  and \emph{exploration} by pulling sub-optimal arms that could help to improve the reward estimation and identify the best arm. 

%Multi-armed Bandit (MAB) \cite{lai1985asymptotically} is a principled solution for the \emph{exploration-exploitation} dilemma in sequential decision-making. In each iteration, the system (also referred to as the principal) pulls an arm and observes the corresponding reward from an unknown but fixed distribution.

The traditional multi-armed bandit (MAB) \cite{lai1985asymptotically} research studies the single-party   setting, where the system has a full control over which arm to pull and can trade off exploitation and exploration for long-term optimality. However, in many real-world applications, such as recommender systems and e-commerce platforms, one often faces a \emph{two-party} game between the system and its users, and the two parties have \emph{different} interests. The system aims at maximizing the long-term reward by recommending exploratory arms; but it cannot directly pull the arm to receive the reward. On the other hand, the arm can only be pulled by the \emph{myopic} users, who seek to maximize their short-term utilities. 
%The users tend to exploit the arm with the best estimated reward, 
This leads to the problem of under-exploration and selection bias: the best arm may remain unexplored forever if it appears sub-optimal initially.
To align the two parties' interest, the system should offer compensations to the users so that the users are motivated to pull the exploratory arm and maximize the long-term reward. This problem is known as \emph{incentivizied exploration}~\cite{kremer2014implementing,frazier2014incentivizing,mansour2015bayesian}. 
%\hf{add a few ciations here} %\huazheng{Give an real-world application such as Amazon review / Airbnb? \cite{hirnschall2018learning}}

Incentivized exploration has been studied in the MAB setting, where the system's goal is to balance the  trade-off among \emph{exploration, exploitation and compensation}, i.e., minimizing total payments while maximizing cumulative rewards \cite{frazier2014incentivizing,hirnschall2018learning,Wang2018MultiarmedBW}. 
Previous solutions assume both the users and the system have access to the same information and both parties maintain the same reward estimation. This assumption is necessary for the system to compute the compensation based on the users' estimated reward difference between the currently best arm and the exploratory arm. Under MAB setting, this assumption naturally holds because both parties observe the same reward feedback and estimate with averaged reward. However, under the contextual bandit setting \cite{Auer02,li2010contextual,Improved_Algorithm}, both parties observe the same rewards but may access different context features. This would lead to different reward estimation and convergence. For example, the users could access the features related to their own private information, which are not accessible by the system. An extreme case in a finite arm setting is that the system may only observe the indices of the arm (which degenerates to the non-contextual MAB), while the users employ informative feature representations of the arms. This representation asymmetry is what we call the \emph{information gap} between the two parties.
This gap leads to different reward estimation between the two parties and brings in the new challenges to  incentivized exploration. For example,  it is even unclear which arm is currently the best on the user side.

%We propose to study the problem of \emph{incentivizing exploration with information gap} in the linear contextual bandit setting.
%Linear bandits~\cite{li2010contextual,Improved_Algorithm} is an important variant of contextual MAB~\cite{Auer02} where the expected reward of an arm is the linear mapping of the context features associated with the arm and an unknown bandit parameter. In our problem, the two parties observe different contexts and are associated with different bandit parameters, which are mapped to the same reward for the same arm. The information gap (representation asymmetry) leads to different parameter and reward estimation for the two parties. Informally, we consider the users' contexts are more informative by assuming they belong to a lower-dimensional subspace, which means %there are fewer parameters to be estimated and 
%the reward estimation on the users side converges faster than the estimation on the system side (see Sec \ref{sec:problem} for details).  
%The main challenge in our setting is that because the two parties observe different contexts, the system lacks the knowledge of the users' estimated reward and thus cannot directly compensate with the difference of users' estimated reward between the current best arm and exploratory arm, which is the \emph{minimum required amount} of incentive. 

In this paper, we study the problem of incentivized exploration in linear contextual bandits under information gap. 
%We first characterize how to incentivize exploration \emph{without} information gap in linear bandits and then propose a solution to incentivize exploration \emph{with} information gap. Without information gap, the system can calculate the minimum amount of compensation for the users to explore a desired arm. 
We proposed an algorithm that incentivizes the user to explore according the Linear UCB strategy~\cite{li2010contextual,Improved_Algorithm}.
The key idea to conquer information gap is that although the system suffers from an information disadvantage and cannot compute the minimum compensation precisely, offering a larger amount of compensation guarantees sufficiency for users to explore. And this added compensation should shrink fast enough such that the total compensation is still sublinear.
%While offering arbitrarily large compensation at every iteration is sufficient to incentivize exploration, the system also need to minimize the total cost. Our proposed method achieved an optimal trade-off between exploration, exploitation and compensation. The system provides compensation related to its confidence interval of the desired exploratory arm. 
%We showed that our compensation strategy is effective to incentive the users and results in sublinear regret. Since the confidence interval is shrinking over time, the total compensation is also sublinear. 
We prove that our algorithm achieves compensation and regret both in the order of $O(d_v \sqrt{T}\log T)$ with information gap and $O(d_x \sqrt{T}\log T)$ without information gap, where $d_x$ and $d_v$ are the dimensions of context features used by the users and the system, respectively. The results suggest that incentivized exploration is still possible with information gap, and the cost of the information gap is realized by the extra compensation that dominated by $d_v$. 
%\hf{rephrase: However, the larger factor $d_v(>d_x)$ resembles the additional cost that the system has to pay due to being at an informationally disadvantaged situation.}
%Add compensation lower bound result later
We also proved the compensation lower bound of incentivized exploration in linear bandits, which recovers the result of compensation lower bound in non-contextual bandits reported in~\citet{Wang2018MultiarmedBW}.
Our simulation-based empirical studies also validate the effectiveness and cost-efficiency of the proposed algorithm.

%\hf{One main question --- do we assume a single user or assume each round has a different but myopic user? If a single user, need to justify why the user is myopic though since if the user shows up for many times, why he is myopic, while not directly run LinUCB? If the user is different at each round, he will indeed by myopic, but then question is why all users see the same features and rewards? Nevertheless, the later story might be easier to justify (e.g., all users care and can observe similar things about a recommendation)?}

% However in practice the system and the user may access different information, which leads to different reward estimation. One scenario is that under the contextual bandit setting, the myopic user has access to the context vectors on each arm and employs linear regression for reward prediction. The user tends to pull the arm that has the highest estimated reward. On the other hand, the system cannot access the arm features and thus can only estimation reward independently for each arm and use stochastic bandits to compute incentives. The goal of the system is to incentivize the user to explore, i.e., pulling arms following a bandits algorithm such as UCB1. Since the user can access side information but the system cannot, the two parties maintain different reward estimation because of the information gap. we seek to answer the following questions: 1) how to incentivize exploration with  information disadvantage on the system side, and 2) how much extra compensation is paid because of the information gap.
\section{Problem Definition}\label{sec:problem}
\textbf{Notations and assumptions}. We study the problem under a linear bandit setting, where a myopic user sequentially interacts with the system for $T$ rounds. At each round $t$, the user observes compensation offered by the system, and pulls an arm $a_t$ from a given arm set $\mathcal{A}_t$. Both the system and the user observe the resulting reward $r_{a_t, t}$ and update their estimations accordingly.

In a contextual bandit setting, each arm $a$ is associated with a context feature vector.  In our problem, for arm $a \in \mathcal{A}_t$, the system observes a feature $\bv_a$ from a $d_v$-dimensional subspace and the users observes a feature $\bx_a$ from a $d_x$-dimensional subspace. Without loss of generality, we will assume $\bx_a \in \mathbb{R}^{d_x}$ and $\bv_a \in \mathbb{R}^{d_v}$ --- if not,   the standard PCA technique can be used to reduce the feature dimensions to $d_x, d_v$~\cite{lale2019stochastic}.
Essentially we consider the features span the whole vector space respectively, which means there is no redundant feature on both sides and the dimensionality cannot be further reduced.   % differentiate dimension from length.   Add a footnote, refer to Section~\ref{sec:related_work} for low rank bandits.

\begin{assumption}[Information Gap] \label{assumption}
There exists a linear transformation $P \in \mathbb{R}^{d_x\times d_v}$ such that for any arm $a$, 
\begin{equation}\label{eq:assumption}
\bx_a = P\bv_a   
\end{equation}
where $d_v \geq d_x$. %\hf{use $d_v > d_x$? Explain that the interesting setting will be $d_v >> d_x$. }
\end{assumption}
% \huazheng{Should we argue the existence of $P$ indicates that $d_v \geq d_x$?}
%
The assumption on $d_v \geq d_x$, i.e., features used by the user belong to a lower dimension space is motivated by many real-world scenarios: for example, users can construct features related to their private information (e.g., age, gender, income or health). %, which lead to a low-dimensional representation.
A notable special case of linear bandits with information gap is a $K$-armed contextual bandit problem, where the system knows nothing beyond   the indices of arms. In this case, the system has no choice but to set the context features as $K$-dimension basis vectors, whereas the user can observe  low-dimensional informative feature representations of the arms with $d_x \ll K$.

The information gap between the two parties is characterized by matrix $P$. 
The linear transformation assumption is to guarantee the two parties face a linear reward mapping, which we stated below.

\textbf{Examples of information gap}. We discussed an extreme case in the introduction where the system is not allowed to access any arm feature except the indices of arms. In this case, the context vectors used by the system are the $K$-dimension one-hot vectors, while the user may observe and employ $d_x$-dimension feature representations of the same arms. The information gap ($K>d_x$) is encoded in the transformation matrix $P$. Now let us consider a less extreme example. Some features could be the combination of both the user's information and item's property, e.g., joint of user's income and the item's price, or joint of user's gender and the item's category. This is a typical way to construct features in the practical recommender systems. The users can employ these informative features and enjoy faster convergence. The system will suffer if it cannot access users' private information. In this example, the transformation matrix $P$ contains the private information hidden from the system.

Note that having access to more features is not equivalent to have more informative representations. Another practical example is that the context vectors used by the system may include many useless or redundant features, where the corresponding weights in the model parameter $\btheta_v^*$ are zeros, i.e., a sparse regression setting. The information gap is captured by the transformation matrix $P$ where the corresponding columns are zero vectors. In this example, the system's features are clearly less informative, i.e., $d_v > d_x$, because of the useless  features.  %This leads to a slower convergence of parameter estimation and a wider confidence interval of reward estimation on the system side, which is the key challenge solved in our paper for incentivized exploration. 

\textbf{Reward mapping}. Following a linear bandit setting, the expected reward of arm $a$ is determined by the inner product between the context features and an unknown bandit parameter. From the user's perspective, we have
\begin{equation*}
    \bbE[r_{a}] = \bx_{a}^\mt \btheta_x^*
\end{equation*}
where $\btheta_x^*$ is the unknown model parameter on the user side.

Based on Assumption~\ref{assumption}, we have $\bx_{a}^\mt \btheta_x^* = \bv_{a}^\mt P^\mt\btheta_x^*$, which suggests there always exists a parameter $\btheta_v^* = P^\mt\btheta_x^*$ on the system side satisfying the same linear reward mapping.
%and $\btheta_v^*$ is the preference parameter . 
We summarize the reward mapping on the two sides as follow:
\begin{equation}
    \bbE[r_{a}] = \bx_{a}^\mt \btheta_x^* = \bv_{a}^\mt \btheta_v^*
\end{equation}

After the user pulls arm $a_t$, both sides observe the reward $r_{a_t, t}$, as 
\begin{equation} \label{eq:reward}
 r_{a_t, t} = \bbE[r_{a_t}] + \eta_t
\end{equation}
where $\eta_t$ is $R$-sub-Gaussian noise.
Without loss of generality, we assume that the norm of the features and parameters are bounded as $\Vert \bx_a \Vert_2 \leq \Vert \bv_a \Vert_2 \leq 1, \Vert \btheta^*_x \Vert_2 \leq 1, \Vert\btheta^*_v \Vert_2 \leq 1$, which naturally bounds the expected reward in the range of $[-1, 1]$ and simplifies the analysis.
Note that the assumption of $\Vert \bx_a \Vert_2 \leq \Vert \bv_a \Vert_2$ is equivalent as assuming the largest singular value of $P$ is upper bounded by 1. Intuitively, this means the linear transformation does not amplify the magnitude of the features. One can always find the satisfying $\bx_a$ by re-scaling $\btheta^*_x$ accordingly.

% \begin{assumption}[Assumption of norm] \label{assumption:norm}
% We assume that the norm of the features and parameters are bounded as $\Vert x_a \Vert_2 \leq \Vert v_a \Vert_2 \leq 1, \Vert \btheta^*_x \Vert_2 \leq 1, \Vert\btheta^*_v \Vert_2 \leq 1$.
% \end{assumption}

The system and the user estimate their own model parameters using ridge regression separately, denoted as $\hat\btheta_{v,t}$ and $\hat\btheta_{x,t}$, by the same observed rewards $\{r_{a_t, t}\}$ but different context features. As a result, the two parties would predict different rewards for the same arm $a$, denoted as $\hat r_{x,a,t} = \bx_{a}^\mt \hat\btheta_{x,t}$ and $\hat r_{v,a,t}=\bv_{a}^\mt \hat\btheta_{v,t}$.
Note that since both feature sets $\{\bx_a\}$ and $\{\bv_a\}$ can generate the same rewards, $d_v \geq d_x$ suggests that features in $\{\bx_a\}$ can better characterize the reward mapping, thus more \emph{informative}. 
The less informative features lead to a slower convergence of the parameter estimation and a wider confidence interval of the reward estimation. Such an information gap brings in new challenges of incentivized exploration.

\textbf{Objective}. The users and the system have different objectives in this sequential decision making problem: the user aims to maximize his/her short-term instantaneous reward, while the system aims to maximize the long-term cumulative reward. At each round $t$, without any incentive, the myopic user will exploit the arm with the highest estimated reward, i.e., $a = \argmax_{i \in \mathcal{A}_t} \hat r_{x,i,t}$. It is well known that the exploitation-only decisions will lead to sub-optimal cumulative reward in the long term. In order to balance exploitation and exploration, the system has to provide compensations to encourage the user to explore. Specifically, the system offers compensation $c_{a,t}$ for pulling arm $a$. Given the incentives, the users maximize the instantaneous utility by pulling arm $a_t = \argmax_{i \in \mathcal{A}_t} \hat r_{x,i,t} + c_{i,t}$.

The system seeks to maximize the cumulative reward, or equivalently, minimize the \emph{cumulative regret} while also minimizing the \emph{total compensation} in expectation. The regret is defined as
\begin{equation}
\label{eq:pregret}
R(T) = \sum_{t=1}^T\big(\bbE[r_{a_t^*}] -\bbE[r_{a_t}]\big) 
%\sum_{t=1}^T(\bx^{\mt}_{a^*_t}\btheta^* - \bx^{\mt}_{a_t}\btheta^*),
\end{equation}
where $a_t^*$ is the optimal arm with the highest expected reward at time $t$. The total compensation is defined as
\begin{equation}\label{eq:pcomp}
C(T) = \sum_{t=1}^T \bbE[c_{a_t, t}]
\end{equation}
An effective incentivized exploration method should balance the trade-off among exploration, exploitation and compensation to obtain \emph{sublinear} cumulative regret and \emph{sublinear} total compensation.

\section{Method}

%We first formally present how to incentivize exploration without information gap when the system explores according to the Linear UCB strategy~\cite{li2010contextual,Improved_Algorithm}. Then we discuss the new challenges and our solution on incentivized exploration under information gap.
We present our solution on incentivized exploration under information gap when the system explores according to the Linear UCB strategy~\cite{li2010contextual,chu2011contextual, Improved_Algorithm}. Then we show that the solution can be easily adopted to the simpler problem setting of incentivized exploration without the information gap. % and discuss the challenges introduced by the information gap. 

\begin{algorithm}[tb]
\caption{Incentivized LinUCB under Information Gap}\label{alg}
\begin{algorithmic}
    \STATE \textbf{Inputs:} $\lambda, \delta$
    \STATE \textbf{Initialize: $\bA_x = \lambda \bI_{d_x}, \bA_v= \lambda \bI_{d_v}, \bb_x = 0, \bb_v = 0$}
    \FOR{ $t=1$ to $T$}	
    \STATE System and user observe context vectors $\{\bx_a\}_{a\in \mathcal{A}_t}$ and $\{\bv_a\}_{a\in \mathcal{A}_t}$ respectively
    \STATE System calculates compensation $c_{a, t}$ for arm $a$ according to Eq~\eqref{eq:incentive}
    \STATE User pulls arm $a_t = \argmax_{a \in \mathcal{A}} \hat r_{x,a,t} + c_{a,t}$
    \STATE System and user observe reward $r_{a_t}$ 
    \STATE // Update on the system side:
    \STATE $\bA_{v, t+1} \gets \bA_{v, t} +   \bv_{a_t} \bv_{a_t}^\mt $, $\bb_{v, t+1} \gets \bb_{v, t} +   \bv_{a_t}  r_{a_{t}} $ 
    \STATE $\hat{\btheta}_{v, t+1} \gets {\bA_{v, t+1}}^{-1} \bb_{v, t+1}$
    \STATE // Update on the user side:
    \STATE $\bA_{x, t+1} \gets \bA_{x, t} +   \bx_{a_t} \bx_{a_t}^\mt $, $\bb_{x, t+1} \gets \bb_{x, t} +   \bx_{a_t}  r_{a_{t}} $ 
    \STATE $\hat{\btheta}_{x, t+1} \gets \bA_{x, t+1}^{-1} \bb_{x, t+1}$    
    \ENDFOR
\end{algorithmic}
\end{algorithm}

\subsection{Incentivized exploration under information gap}

We present Algorithm~\ref{alg} to show how the system incentivizes the myopic user to follow the desired exploration strategy under information gap. At each round, the system and the user observe context features $\{\bx_a\}$ and $\{\bv_a\}$ respectively for the same arm set $\mathcal{A}_t$. The system needs to motivate the user to explore arm $a_t$ according to LinUCB strategy based on its current parameter estimation $\hat\btheta_{v,t}$. To incentivize the user to pull arm $a_t$, the system offers compensation $c_{a_t,t}$ according to Eq~\eqref{eq:incentive}. Note that the system does not offer incentives to the other arms and sets $c_{i,t} = 0, \forall i \neq a_t$. The myopic user pulls the arm that maximizes the sum of his/her estimated reward $\hat r_{x,a,t}$ and the compensation $c_{a,t}$ In Lemma~\ref{lemma:incentive} we guarantee that the user will pull the system desired arm $a_t$. Both the system and the user then observe reward feedback $r_{a_t}$, and update their parameters using ridge regression accordingly. 

Denote $\textit{CB}_{x,t}(\bx_a)$ as the width of the user's confidence interval of arm $a$ at time $t$, which is computed as
\begin{equation*}
   \textit{CB}_{x,t}(\bx_a) = \alpha_{x,t} \Vert \bx_a \Vert_{A_{x,t}^{-1}} 
\end{equation*}
where
\begin{equation*}
\alpha_{x,t} = R\sqrt{d_x \log \frac{1+t/\lambda}{\delta}} +\sqrt{\lambda}
\end{equation*} %following Theorem 2 of~\cite{Improved_Algorithm}.%,

The value of $\alpha_{x,t}$ is the upper bound of the width of confidence ellipsoid and is set according to the following lemma.
\begin{lemma}[Theorem 2 of~\cite{Improved_Algorithm}]\label{lemma:define_cb}
With  probability at least $1-\delta$, the parameter $\btheta_x^*$ lies in the confidence ellipsoid of $\hat\btheta_{x,t}$ satisfying 
\begin{equation*}
    \Vert \hat\btheta_{x,t} - \btheta_x^* \Vert_{A_{x,t}}  \leq \alpha_{x,t}
\end{equation*}
for all $t \geq 0$.
\end{lemma}

Similar to $\textit{CB}_{x,t}(\bx_a)$, we denote the width of confidence interval on the system side as 
\begin{equation*}
   \textit{CB}_{v,t}(\bv_a) = \alpha_{v,t} \Vert \bv_a \Vert_{A_{v,t}^{-1}} 
\end{equation*}
where
\begin{equation*}
\alpha_{v,t} = R\sqrt{d_v \log \frac{1+t/\lambda}{\delta}} +\sqrt{\lambda}
\end{equation*}

The key challenge in incentivized exploration under information gap is that the system does not maintain the same reward estimation as the user's, because the two sides use different features to learn and predict rewards. This prevents us from computing minimum required compensation and makes the problem non-trivial. We have to carefully determine the compensation: a larger amount of incentives is required to guarantee that user will explore while we also need to keep the incentives small to maintain a sublinear total compensation.
We first use the following lemma to show that on the same arm, the confidence interval by the system's reward estimation is no smaller than the confidence interval by the user's estimate. This lemma guarantees in Algorithm~\ref{alg} the system provides sufficient incentives to the user to pull the arms according to an upper confidence bound type exploration strategy.

\begin{lemma}
\label{lemma:cb} 
Consider two least square estimators (ridge regression) that estimate the model parameters with the same reward observations but different features satisfying Assumption \ref{assumption}. For  all $t \geq 0$ and all arm $a\in \mathcal{A}_t$, we have
\begin{equation}
\textit{CB}_{v,t}(\bv_a) \geq \textit{CB}_{x,t}(\bx_a),
\end{equation}
i.e., the confidence interval maintained on the system side is no smaller than the user side estimation.
\end{lemma}

\begin{proof}[Proof Sketch.]
Since $\textit{CB}_{v,t}(\bv_a) = \alpha_{v,t}\lVert\bv_a\rVert_{\bA_{ v,t}^{-1}}$ and $\textit{CB}_{x,t}(\bx_a) = \alpha_{x,t}\lVert\bx_a\rVert_{\bA_{ x,t}^{-1}}$, we can prove $\lVert\bv_a\rVert_{\bA_{v,t}^{-1}} \geq \lVert\bx_a\rVert_{\bA_{x,t}^{-1}}$ and $\alpha^v_{t} \geq \alpha^x_{t}$ separately. It is obvious that $\alpha^v_{t} \geq \alpha^x_{t}$ because $d_v \geq d_x$. Substitute $\bx_a = P\bv_a$ and we can %use Schur complement to 
prove that $ \bA_{v,t}^{-1} -  P^\mt \left({P \bA_{v,t} P^\mt}\right)^{-1} P$ is a positive semi-definite matrix, which leads to $\lVert\bv_a\rVert_{\bA_{ v,t}^{-1}} \geq \lVert\bx_a\rVert_{\bA_{x,t}^{-1}}$.
\end{proof}

The intuition behind this lemma is straightforward. The confidence interval characterizes the uncertainty of reward prediction. Since the estimator on the users side uses more informative features, its parameter estimation converges faster and its confidence interval is smaller than that maintained on the system side.

Based on Lemma~\ref{lemma:cb}, we have the following lemma,
\begin{lemma}\label{lemma:incentive} 
For all $t \geq 0$, with probability at least $1-2\delta$, the users are incentivized to pull the desired arm with compensation
\begin{equation}\label{eq:incentive}
c_{a_t, t} = 4\textit{CB}_{v,t}(\bv_{a_t}) 
\end{equation}
 to arm \begin{equation}\label{eq:relaxed_UCB}
 a_t = \argmax_a \left(\bv_a^\mt\hat\btheta_{v,t} + 2\textit{CB}_{v,t}(\bv_a) \right),
 \end{equation}
i.e., the arm with the highest (relaxed) upper confidence bound according to the system's estimate. 
\end{lemma}

\begin{proof}
In order to incentivize the user to pull  arm $a_t$, the \emph{minimum required compensation} is $\max_i \hat r_{x,i,t} - \hat r_{x, a_t, t}$. However, since the system cannot access the context features the user uses and thus maintains different reward estimates, it has to provide compensation larger than the minimum required amount. 

Denote the user's greedy choice as $g = \argmax_i \hat r_{x, i, t}$.
To show that $c_{a_t,t}$ is sufficient, we need to prove that  the user prefers the exploratory arm $a_t$ with compensation over his/her greedy choice, i.e., $\hat r_{x, g, t} \leq \hat r_{x, a_t, t} + c_{a_t,t}$.

Based on Lemma~\ref{lemma:define_cb}, we have that for all $t \geq 0$, with probability at least $1-\delta$, we have
$\vert\hat r_{x, a, t} -\bbE[r_{a}]\vert \leq \textit{CB}_{x,t}(\bx_a)$ and $\vert\hat r_{v, a, t} -\bbE[r_{a}]\vert \leq \textit{CB}_{v,t}(\bv_a)$ 
hold for any arm $a$ at any time $t$. Using the union bound, with probability at least $1-2\delta$ we have
\begin{align} 
\label{eq:reward_diff}
\vert\hat r_{x, a, t} -\hat r_{v, a, t}\vert
&\leq \vert\hat r_{x, a, t} -\bbE[r_{a}]\vert + \vert \bbE[r_{a}] - \hat r_{v, a, t}\vert\nonumber\\
&\leq \textit{CB}_{x,t}(\bx_a) + \textit{CB}_{v,t}(\bv_{a}) 
\end{align}
Then we can bound the user's reward estimate from the system side as follows,
\begin{align}\label{eq:derive_comp}
        \hat r_{x, g, t} & \leq \hat r_{v, g, t} + \textit{CB}_{x,t}(\bx_g) + \textit{CB}_{v,t}(\bv_{g})\nonumber\\
        & \leq \hat r_{v, g, t} + 2\textit{CB}_{v,t}(\bv_{g}) \nonumber\\
        & \leq \hat r_{v, a_t, t} +  2\textit{CB}_{v,t}(\bv_{a_t})\nonumber\\
        &  \leq \hat r_{x, a_t, t} + \textit{CB}_{x,t}(\bv_{a_t}) + \textit{CB}_{v,t}(\bv_{a_t}) + 2\textit{CB}_{v,t}(\bv_{a_t}) \nonumber\\
        &  \leq \hat r_{x, a_t, t} + 4\textit{CB}_{v,t}(\bv_{a_t})
\end{align}
The first and fourth steps are based on Eq~\eqref{eq:reward_diff}. The second and last  steps are based on Lemma~\ref{lemma:cb}. The third inequality is based on the UCB strategy in Eq~\eqref{eq:relaxed_UCB}.
% The last step is based on Lemma~\ref{lemma:cb}.
\end{proof}

\begin{algorithm}[tb]
\caption{Incentivized LinUCB without Information Gap}\label{alg:no_gap}
\begin{algorithmic}
    \STATE \textbf{Inputs:} $\lambda, \delta$
    \STATE \textbf{Initialize:} $\bA_x = \lambda \bI, \bb_x = 0$
    \FOR{ $t=1$ to $T$}	
    \STATE System and user observe context vectors $\{\bx_a\}_{a\in \mathcal{A}_t}$
    \STATE System calculate compensation $c_{a, t}$ for arm $a$ according to Eq~\eqref{eq:incentive_no_gap}
    \STATE User pulls arm $a_t = \argmax_{a \in \mathcal{A}} \hat r_{x,a,t} + c_{a,t}$
    \STATE System and user observe reward $r_{a_t}$ 
    % \STATE // Update on system side:
    % \STATE $\bA_{v, t+1} \gets \bA_{v, t} +   \bv_{a_t} \bv_{a_t}^\mt $, $\bb_{v, t+1} \gets \bb_{v, t} +   \bv_{a_t}  r_{a_{t}} $ 
    % \STATE $\hat{\btheta}_{v, t+1} \gets {\bA_{v, t+1}}^{-1} \bb_{v, t+1}$
    % \STATE // Update on user side:
    \STATE $\bA_{x, t+1} \gets \bA_{x, t} +   \bx_{a_t} \bx_{a_t}^\mt $, $\bb_{x, t+1} \gets \bb_{x, t} +   \bx_{a_t}  r_{a_{t}} $ 
    \STATE $\hat{\btheta}_{x, t+1} \gets \bA_{x, t+1}^{-1} \bb_{x, t+1}$    
    \ENDFOR
\end{algorithmic}
\end{algorithm}

It is worth noting that the system follows a more optimistic arm selection strategy in Eq~\eqref{eq:relaxed_UCB} using a confidence interval twice larger than the classical LinUCB algorithm's. We follow this relaxed upper confidence bound because we need to consider the uncertainty on both parties as the first step of the derivation in Eq~\eqref{eq:derive_comp} suggested. It is unclear whether we can incentivize the user to follow the classical LinUCB algorithm. Intuitively, our exploration strategy results in a twice larger regret than the classical LinUCB's, which is still in the same order for $T$. We provide the regret and compensation upper bound of Algorithm~\ref{alg} in Section~\ref{sec:analysis}. 

\subsection{Incentivized exploration without information gap}

Our solution can be easily adopted to solve the incentivized exploration problem of without information gap.
In Algorithm~\ref{alg:no_gap}, we show how the system incentivizes the myopic user to follow the desired exploration strategy in this simpler setting. 

Without information gap, the system and the user maintain the same parameter and reward estimations, and the \emph{minimum required compensation} to incentivize the user to explore according to LinUCB equals to the difference of estimated rewards between the currently best arm and the exploratory arm. The system thus only needs to offer compensation by,
\begin{equation}\label{eq:incentive_no_gap}
c_{a_t, t} = \max_i \hat r_{x,i,t} - \hat r_{x,a_t,t}
\end{equation}
to arm $a_t = \argmax_a \left(\bx_a^\mt\hat\btheta_{x,t} + \textit{CB}_{x,t}(\bx_a)\right)$.
%, where $a_t$ is the arm with the highest upper confidence bound. 
%The system does not offer incentives to other arms. 
The user will pull the exploratory arm, because $a_t = \argmax_i \hat r_{x,i,t} + c_{i, t}$, i.e., arm $a_t$ can maximize user's instantaneous utility.

Since Algorithm \ref{alg:no_gap} guarantees that  the  user is incentivized to pull arms according to LinUCB, its regret is the same as LinUCB's in the order of $O(d_x\sqrt{T}\log T)$ (see Theorem 3 of \cite{Improved_Algorithm}). Its compensation upper bound is stated below.
 
\begin{theorem} [Compensation upper bound without information gap] \label{theorem:comp_no_gap}
With probability at least $1-\delta$, the total compensation provided in Algorithm \ref{alg:no_gap} is upper bounded as
\begin{equation*}
C(T) \leq 
\left(R\sqrt{d_x \log \frac{1+T/\lambda}{\delta}} +\sqrt{\lambda}\right)\sqrt{Td_x\log(\lambda + \frac{T}{d_x})}
\end{equation*}
\end{theorem}
\begin{proof}[Proof Sketch.]
We can first show that with a high probability the compensation at round $t$ is upper bounded by the confidence interval, i.e.,  $c_{a_t, t} \leq \textit{CB}_{x,t}(\bx_{a_t})$. Then the total compensation can be upper bounded by $\sum_t \textit{CB}_{x,t}(\bx_{a_t})$, which can be bounded using Lemma 11 of \citet{Improved_Algorithm}.  
\end{proof}

% The proof is simple: we can bound $c_{a_t, t} \leq \textit{CB}_{x,t}(\bx_{a_t})$; then the upper bound of total compensation is $\sum_t \textit{CB}_{x,t}(\bx_{a_t})$ the same as the regret bound.

Note that without information gap, both the regret and compensation upper bounds are in the order of $O(d_x\sqrt{T}\log T)$, with a linear dependency on the feature dimension $d_x$. 
%We will show in Section~\ref{sec:analysis} that because of the information gap, the regret and compensation will increase to $O(d_v\sqrt{T}\log T)$.

\textbf{Discussion}.   
Without information gap, i.e., the two parties have access to the same features and maintain the same reward predictions, the system can offer the minimum required compensation as shown in Eq~\eqref{eq:incentive_no_gap} to incentivize exploration. 
With information gap, compensate by Eq \eqref{eq:incentive} can still successfully incentivize exploration in a high probability manner, but it is inevitably larger than the minimum amount. More specifically, without information gap the required compensation can be computed deterministically in Eq~\eqref{eq:incentive_no_gap}; otherwise, the system can only estimate the reward difference with a high probability (as shown in Lemma~\ref{lemma:incentive}). We also notice without information gap the system does not compensate if the greedy choice also has the largest upper confidence bound, which happens more often in the later rounds when the reward estimation converges. But with information gap, our algorithm always compensates, because $\textit{CB}_{v,t}(\bv_{a_t}) > 0$, i.e., the system does not know if the user's greedy choice is also preferred in terms of its UCB. We will show in the next section that the total compensation is still sublinear under information gap.

\section{Analysis}\label{sec:analysis}

We first analyze the regret and compensation upper bound of Algorithm~\ref{alg}. We then discuss the compensation lower bound of the problem.

\subsection{Regret and compensation upper bound}

\begin{theorem} \label{theorem:regret}
With probability at least $1-3\delta$, the cumulative regret of Algorithm \ref{alg} is upper bounded by
\begin{equation*}
R(T) \leq 
\left(2R\sqrt{d_v \log \frac{1+T/\lambda}{\delta}} +\sqrt{\lambda}\right)\sqrt{Td_v\log(\lambda + \frac{T}{d_v})}
\end{equation*}
\end{theorem}

Theorem~\ref{theorem:regret} shows that the cumulative regret of Algorithm~\ref{alg} is in the order of $O(d_v \sqrt{T\log T})$. The proof mostly follows the regret analysis of LinUCB, though we have to use a wider confidence interval for exploration. Note that the resulting probability is $1-3\delta$, because the users will follow the system's exploration strategy with probability at least $1-2\delta$ as shown in Lemma~\ref{lemma:incentive} and the confidence bound holds with probability at least $1-\delta$.

\begin{theorem} \label{theorem:comp}
With probability at least $1-2\delta$, the total compensation provided in Algorithm \ref{alg} is upper bounded by
\begin{equation*}
C(T) \leq 
\left(4R\sqrt{d_v \log \frac{1+T/\lambda}{\delta}} +\sqrt{\lambda}\right)\sqrt{Td_v\log(\lambda + \frac{T}{d_v})}
\end{equation*}
\end{theorem}

Theorem~\ref{theorem:comp} shows that the total compensation of Algorithm~\ref{alg} is in the order of $O(d_v \sqrt{T\log T})$. Combining Theorem~\ref{theorem:regret} and~\ref{theorem:comp} we showed that our proposed algorithm can incentivize exploration under information gap and achieve sublinear regret and compensation. We notice that the two upper bounds linearly depend on the system's feature dimension $d_v$. Comparing to the no information gap setting where we showed both the regret and compensation is in the order of $O(d_x \sqrt{T\log T})$, the added regret and compensation are $O((d_v-d_x)\sqrt{T\log T})$. And the corresponding high probability guarantee drops a little. These results suggest that the complexity/difficulty of the problem is characterized by the dimensionality of the observed context features, exactly where the information gap comes from. %Next we show that the compensation lower bound also linearly depends on dimension.

%Note that the compensation lower bound is applied to when the system expect the users follow its (less efficient) exploration strategy. 

%The reason is that the analysis of Algorithm~\ref{alg} is general and can be applied to linear bandits with infinite arms.  
%\huazheng{In the next section, we wish to show that the regret and compensation can be reduced to $O(\sqrt{d_v T\log (TK)})$ in the $K$-armed linear bandits setting. }

%The regret is in the order of $O(d_v\sqrt{T}\log T)$.

\subsection{Compensation lower bound}

% \huazheng{We wish to show that the lower bound of regret and compensation are $\Omega(d_x\sqrt{ T})$ in the infinite arms setting, and $\Omega(\sqrt{d_xT})$ in the finite arm setting}
We now prove a gap-dependent asymptotic compensation lower bound of incentivized exploration in linear bandits with finite arms, and show that our result recovers the lower bound of incentivized exploration reported in non-contextual bandits in \cite{Wang2018MultiarmedBW}. %(with and without Information Gap)

%Prove a minimax bound? Impossible.
Let $G_{x,T} = \mathbb{E}\left[\sum_{t=1}^T \bx_{a_t} \bx_{a_t}^\mt\right]$. Without loss of generality assume arm 1 is the best arm and $\Delta_{a} =\bbE[r_{1}] - \bbE[r_{a}] = (\bx_1 - \bx_a)^\mt\btheta^*$ is the reward gap between arm $a$ and the best arm .

\begin{theorem}[Compensation lower bound without information gap]\label{theorem:comp_lb_no_gap}
Consider any consistent algorithm observing context features $\{\bx_a\}_{a\in \mathcal{A}}$ that guarantees an $o(T^p)$ regret upper bound for any $T>0$ and $0<p\leq 1$.
%\footnote{This is also known as consistent algorithm in the literature.}.
In order to incentivize a user with a least square estimator of rewards to follow the algorithm's choice, the total compensation $C(T)$ for sufficiently large $T$ is 
\begin{equation*}
    \Omega\left(c_x(\mathcal{A}, \btheta^*)\log(T)\right),
\end{equation*}
% \begin{equation}
%     \limsup_{T \to \infty} \frac{C(T)}{\log(T)} \geq c_x(\mathcal{A}, \btheta^*),
% \end{equation}
where $c_x(\mathcal{A}, \btheta^*)$ is the optimal value of the following optimization problem
\begin{equation}\label{eq:optimization}
    \begin{split}
       &c_x(\mathcal{A}, \btheta^*) = \inf_{\alpha \geq 0} \sum_{\bx_a} \alpha_{\bx_a}\frac{\Delta_{a}}{3}\\
       &\text{s.t. } \Vert \bx_a \Vert^2_{H_{x,T}^{-1}} \leq \frac{\Delta_{a}^2}{2}, \forall \bx_a \text{ with } \Delta_{a} > 0
    \end{split}
\end{equation}
where $H_{x,T} = \sum_{\bx_a} \alpha_{\bx_a}\bx_{a_t} \bx_{a_t}^\mt$.
\end{theorem}

Our proof relies on the following lemmas:
\begin{lemma}[Theorem 1 in~\citet{lattimore2017end}]\label{lemma:minimum_exploration}
Assume $G_{x,T}$ is invertible for sufficiently large $T$. For all suboptimal $a\in\mathcal{A}$ it holds that
\small
\begin{equation*}
    \limsup_{T \to \infty}{\log T \Vert \bx_a -  \bx_1 \Vert^2_{G_{x,T}^{-1}} \leq \frac{\Delta_{a}^2}{2}}
\end{equation*}
\normalsize
\end{lemma}
\begin{lemma}[Theorem 8 in~\citet{lattimore2017end}]\label{lemma:concentration}
For any $\delta \in [1/T, 1)$, $T$ sufficiently large and $t_0$ such that $G_{t_0}$ is almost surely non-singular,
\begin{equation*}
\mathbb{P}\left(\exists t\geq 0, \bx_a:\vert\hat r_{x,a,t} - \bbE[r_a]\vert\geq\sqrt{\Vert\bx_a\Vert^2_{G_{x,t}^{-1}}f_{T,\delta}}  \right)\leq\delta
\end{equation*}
where for some $c > 0$ universal constant
\[
f_{T,\delta} = 2\left(1+\frac{1}{\log (T)}\right)\log(1/\delta)+cd_x\log(d_x\log(T))
\]
\end{lemma}

%The proof idea is similar to the compensation lower bound of incentivized non-contextual bandits in \citet{Wang2018MultiarmedBW}
\begin{proof}[Proof Sketch]
Suppose an algorithm is consistent with regret $o(T^p)$, %the constraint in Eq~\eqref{eq:optimization} is based on Theorem 1 in~\citet{lattimore2017end} and 
Lemma~\ref{lemma:minimum_exploration}
suggests that the algorithm must collect a sufficient number of samples such that the width of the confidence interval is small enough to identify the suboptimal arms. 
Since the algorithm has $o(T)$ regret, we can find $t_1$ such that the best arm is pulled at least $T/2$ times; and because of the concentration result in Lemma~\ref{lemma:concentration}, its confidence interval is smaller than $\Delta_{2}/3$ where $\Delta_{2}$ is the reward gap between the best arm and second best arm. This means for $t>t_1$ we have $\hat r_{x,1, t} \geq \bbE[r_{1}] - \Delta_{2}/3$ with a high probability.

For any other arm $a$, from Lemma~\ref{lemma:minimum_exploration} and the concentration bound we can show that it will also be pulled enough times such that its confidence interval is smaller than $\Delta_{a}/3$ with a high probability after a fixed round $t_a$. Therefore, for $t>t_a$ we have $\hat r_{x, a, t} \leq \bbE[r_{a}] + \Delta_{a}/3$. Combining the two inequalities we know that after a fixed time point, the minimum required compensation to incentivize the user to pull arm $a$ is $\hat r_{x, 1, t} - \hat r_{x, a, t} \geq \Delta_{a}/3$. We then use the optimization problem in Eq~\eqref{eq:optimization} to obtain the compensation lower bound, where the optimization minimizes the total compensation and satisfies the consistent constraints that the gaps of all suboptimal arms are identified with high confidence.% the algorithm is consistent.
\end{proof}

Next, we construct an example to illustrate our lower bound analysis. 

\textbf{Example}. When $\{\bx_a = e_a \in \mathbb{R}^{d_x}\}_{a\in \mathcal{A}}$ are the basis vectors, the problem reduces to a non-contextual $K$-armed bandit with $K=d_x$. By setting $\Vert\bx_a \Vert^2_{H_{x,T}^{-1}} = \Delta_{a}^2/2$, we have $\alpha_{\bx_a} = 2/\Delta_{a}^2$ and $c_x(\mathcal{A}, \btheta^*) = \sum_{a\in \mathcal{A}, \Delta_{a}>0}\frac{2}{3\Delta_{a}}$. This gives us the compensation lower bound as follows,
\begin{equation*}
    C(T)= \Omega\left(\sum_{a\in \mathcal{A}, \Delta_{a}>0}\frac{\log(T)}{\Delta_{a}} \right)
\end{equation*}
This result recovers the lower bound of incentivized exploration in non-contextual bandits in \cite{Wang2018MultiarmedBW}. We also notice that the result can be further bounded as 
\begin{equation*}
    C(T) = \Omega\left(\frac{d_x\log(T)}{\max_{a\in \mathcal{A}}\Delta_{a}} \right),
\end{equation*}
where we observe a linear dependency on dimension $d_x$.

Note that our compensation lower bound is in the order of $\Omega(\log(T))$, because it is a gap-dependent bound. We leave the question of whether one can obtain an $\Omega(\sqrt{T})$ gap-independent compensation lower bound for general infinite arm setting, which will match our upper bound in Theorem~\ref{theorem:comp}, as an open problem.
\begin{figure*}[ht]
\centering
\setlength\tabcolsep{.5pt}
\begin{tabular}{c c c}
%\hspace*{-0.4cm}
\includegraphics[width=5.5cm]{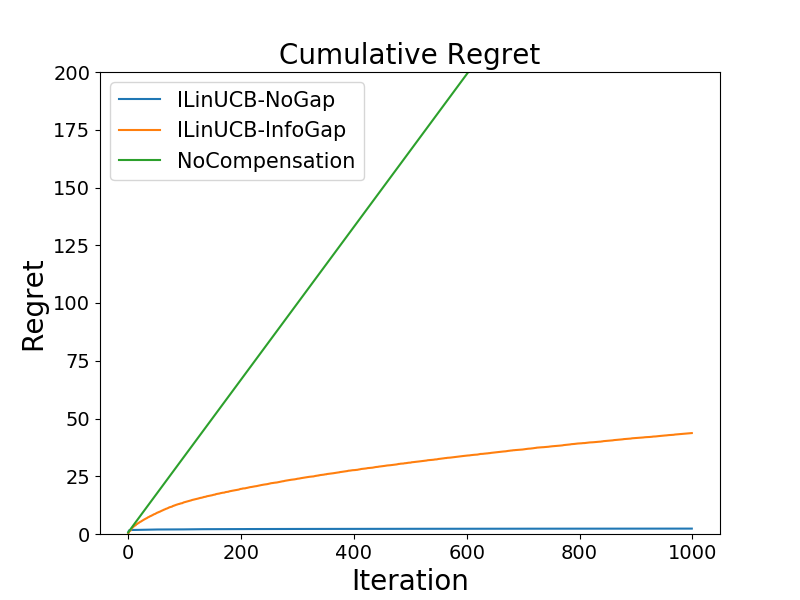} &
\includegraphics[width=5.5cm]{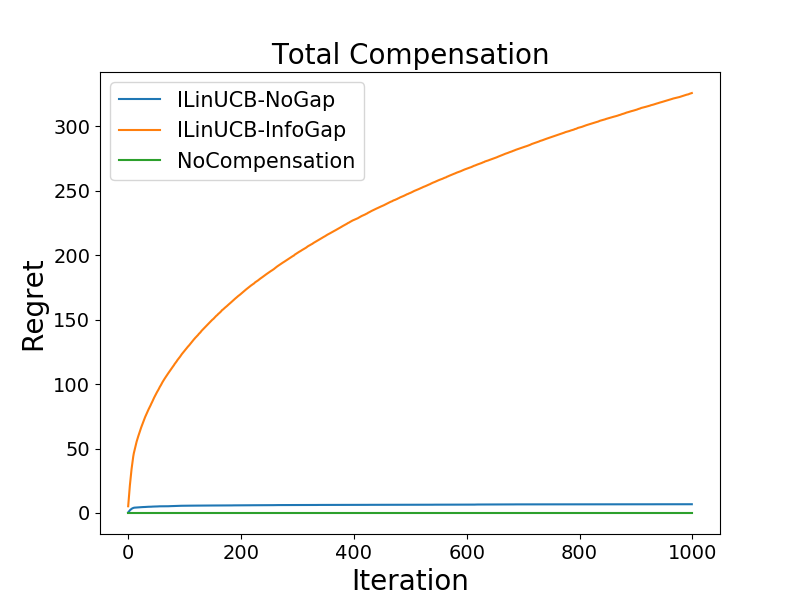}&
\includegraphics[width=5.5cm]{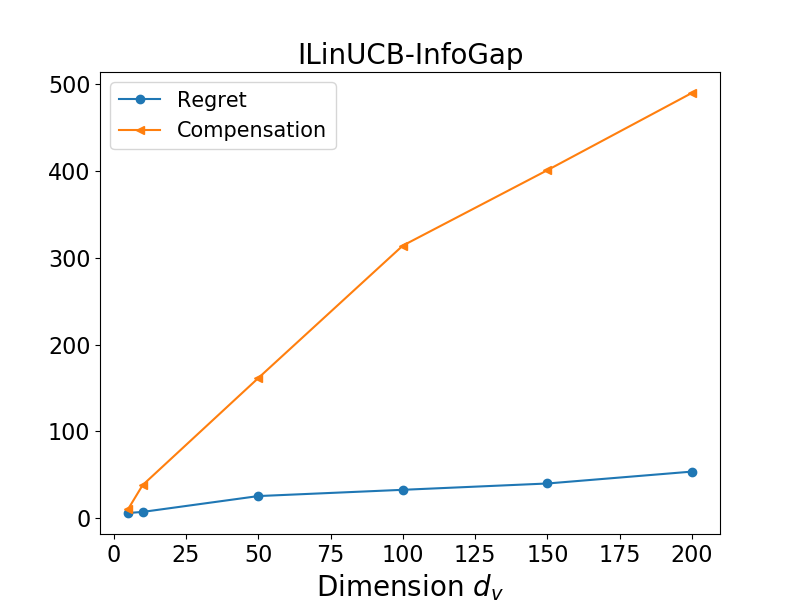}
\\
(a) Regret & (b) Compensation & (c) Varying dimension $d_v$\\
\end{tabular}
\vspace{-1mm}
\caption{Simulation result on randomly sampled features with $d_x=5$ and $d_v=100$}
\label{fig_simu}
\vspace{-2mm}
\end{figure*}

\begin{corollary}[Compensation lower bound under information gap]
Consider any consistent algorithm observing context features $\{\bv_a\}_{a\in \mathcal{A}}$ that guarantees an $o(T^p)$ regret upper bound for any $T>0$ and $0<p\leq 1$. To incentivize the user who observes context features $\{\bx_a\}_{a\in \mathcal{A}}$ satisfying Assumption~\ref{assumption} with a least square estimator, the total compensation $C(T)$ for sufficiently large $T$ is 
\begin{equation*}
    \Omega\left(c_v(\mathcal{A}, \btheta^*)\log(T)\right),
\end{equation*}
where $c_v(\mathcal{A}, \btheta^*)$ is the optimal value of the following optimisation problem
\begin{equation*}
    \begin{split}
      &c_v(\mathcal{A}, \btheta^*) = \inf_{\alpha \geq 0} \sum_{\bv_a} \alpha_{\bv_a}\frac{\Delta_{a}}{3}\\
      &\text{s.t. } \Vert \bv_a \Vert^2_{H_{v,T}^{-1}} \leq \frac{\Delta_{a}^2}{2}, \forall \bv_a \text{ with } \Delta_{a} > 0
    \end{split}
\end{equation*}
where $H_{v,T} = \sum_{\bv_a} \alpha_{\bv_a}\bv_{a_t} \bv_{a_t}^\mt$.
\end{corollary}

The proof of compensation lower bound under information gap mostly follows Theorem~\ref{theorem:comp_lb_no_gap} by simply replacing the user's feature $\bx_a$ with the system's feature $\bv_a$. The main difference is that when applying the concentration bound in Lemma~\ref{lemma:concentration} to derive the minimum required compensation, we still use $\bx_a$ because the minimum amount is based on the \textit{user's} estimated reward difference between the currently best arm and the exploratory arm.
However, we notice that $\bx_a$ or $d_x$ does not directly appear in this lower bound. The impact of $\bx_a$ being in a lower-dimensional space is that we have a faster concentration bound to have the confidence interval smaller than $\Delta_{a}/3$ at an earlier time point. Since we consider $T\to\infty$, this does not change the order of the bound and the final result is dominated by $\bv_a$. 

Considering a similar example of $K$-armed bandit setting where $K = d_v$, we can obtain 
\begin{equation*}
    C(T) = \Omega\left(\frac{d_v\log(T)}{\max_{a\in \mathcal{A}}\Delta_{a}} \right)
\end{equation*}
where we observe a linear dependency on dimension $d_v$.

\section{Experiments}

%\textbf{Simulation setup}.  
We use simulation-based experiments to verify the effectiveness of our proposed incentivized exploration solution. 
In our simulations, we generate a size-$K$ arm pool $\cA$, in which each arm $a$ is associated with a $d_v$-dimension vector $\bv_a$ as the system observed features and a $d_x$-dimension vector $\bx_a$ as the user observed features. Each dimension of $\bv_a$ is drawn from a set of zero-mean Gaussian distributions with variances sampled from a uniform distribution $U(0,1)$. Each $\bv_a$ is then normalized to $\Vert\bv_a\Vert_2 = 1$.
We then sample the elements of the $d_x \times d_v$ transformation matrix $P$ from $N(0, 1)$ and normalize each row $i$ by $\Vert P_i \Vert_2 = 1$. Following Assumption~\ref{assumption}, the user observed features $\bx_a$ are generated as $\bx_a = P\bv_a$. $P$ guarantees that $\Vert\bx_a\Vert_2 \leq \Vert\bv_a\Vert_2 = 1$. User's model parameter $\btheta_x^*$ is sampled from $N(0, 1)$ and normalized to $\Vert\btheta_x^*\Vert_2  = 1$. System's model parameter is set to $ \btheta_v^*= P\btheta_x^*$.
At each round $t$, the same set of arms were presented to all the algorithms, but the system and the user observe different features respectively. After the user pulls an arm $a_t$, both the user and the system observe its reward following Eq~\eqref{eq:reward}. We set $d_x$ to 5, $d_v$ to 100, the standard derivation $\sigma$ of Gaussian noise $\eta_t$ to 0.1, and the arm pool size $K$ to 100 in our simulations.

We compare the following algorithms:
1) ILinUCB-InfoGap: our Algorithm~\ref{alg} where $\{\bv_a\}_{a\in\mathcal{A}_t}$ is observed by the system; 2) ILinUCB-NoGap: our Algorithm~\ref{alg:no_gap} where both the system and the user observe $\{\bx_a\}_{a\in\mathcal{A}}$; 3) NoCompensation: a baseline system that does not offer any compensation to the user. The myopic user always pulls the current best arm. We set the probability $\delta = 0.01$ and regularization coefficient $\lambda=0.1$ for all the algorithms.

We report the averaged results of 10 runs where in each run we sample a random model parameter $\btheta_x^*$.  In Figure~\ref{fig_simu}(a), we observe that without providing any compensation, the myopic user suffers a linear regret, which emphasizes the importance of incentivized exploration. Both ILinUCB-InfoGap and ILinUCB-NoGap enjoy sublinear regret and compensation. The added regret of ILinUCB-InfoGap shows the algorithm explores slower in the large $R^{d_v}$ space because of the information gap.
%We limit the y-axis to show more details on other curves. 

We notice that the total compensation of ILinUCB-InfoGap in Figure~\ref{fig_simu}(b) is sublinear and keeps increasing. The algorithm has to always compensate due to the information gap as we discussed before. ILinUCB-NoGap, however, rarely compensates in the later stage. This is because  when system  explored sufficiently, greedy choice on the user side agrees with the UCB strategy on the system side, and thus no compensation is needed.   %if the user's greedy choice is also preferred in terms of its UCB,
In Figure~\ref{fig_simu}(c), we vary the dimension of system's feature $d_v$ from $5$ to $200$ while fixing $d_x = 5$. We observe that both regret and compensation increases linearly with $d_v$, which confirms our theoretical upper bound. 
%To get a stable result, We fix $\bx_a$ and use pseudo inverse of P to get $bv_a$.

\begin{figure}[ht]
\vspace{-1mm}
\centering
\setlength\tabcolsep{.5pt}
\begin{tabular}{c c}
\hspace*{-0.4cm}
\includegraphics[width=4.5cm]{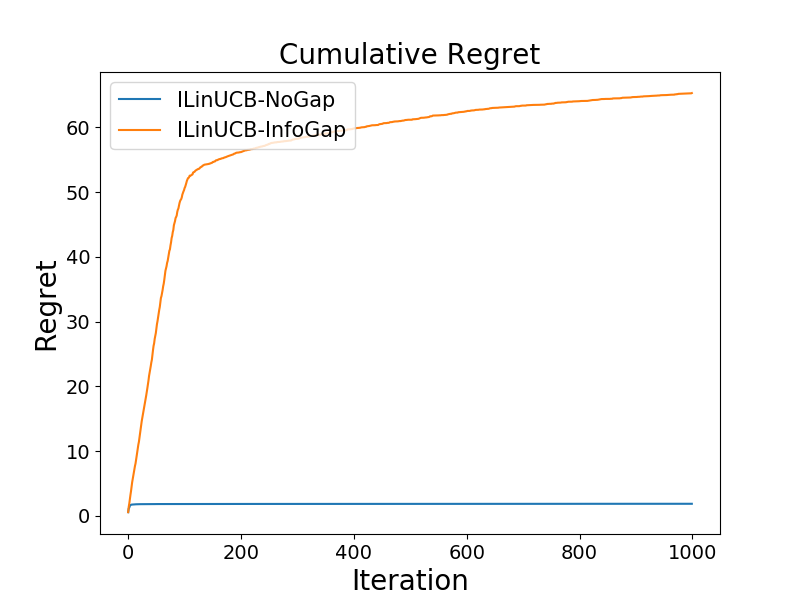} &
\includegraphics[width=4.5cm]{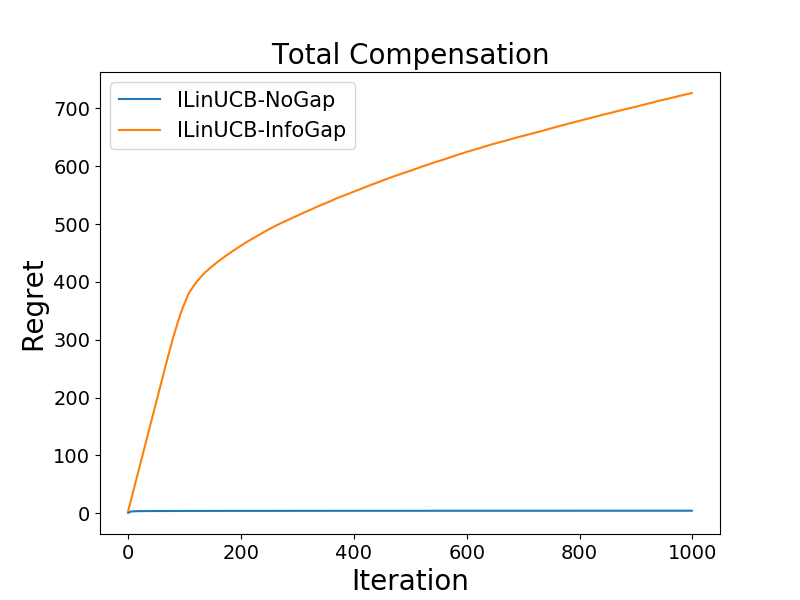}
\\
(a) Regret & (b) Compensation\\
\end{tabular}
\vspace{-1mm}
\caption{MAB setting where the system only observes the indices of the arms.} \label{fig_simu_MAB}
\vspace{-1mm}
\end{figure}
%\textbf{MAB setting}.
In Figure~\ref{fig_simu_MAB}, we simulate a $K$-armed bandit setting where only the indices of the arms are available to the system. The system sets $\bv_a = e_a \in \mathbb{R}^K$. The rest of the settings are the same as described above. In this setting, our ILinUCB-InfoGap  explores almost equivalently to UCB1~\cite{auer2002finite} and can be viewed as a more optimism version of the Incentivized UCB algorithm in \cite{Wang2018MultiarmedBW} with a wider confidence interval in consideration of the information gap. The system observes the least  information in this setting. We notice that its regret and compensation are much larger than the results in Figure~\ref{fig_simu} where $\{\bv_a\}_{a\in\mathcal{A}}$ is more informative about the rewards. This again confirms that the system inevitably suffers higher regret and compensation when the features are less informative. %Incentivized UCB cannot guarantee the user will always follow because the  We observe that 

\section{Related Work}\label{sec:related_work}

The incentivized exploration in multi-armed bandits has been studied since \cite{kremer2014implementing, frazier2014incentivizing}. See \citet{slivkins2017incentivizing} for an overview. One line of the studies \cite{kremer2014implementing,mansour2015bayesian,immorlica2018incentivizing, sellke2020sample} assume the system has information advantage on observing the full arm-pulling history while users do not. The system leverages the information asymmetry to recommend exploratory arms as long as the users do not have a better choice from their perspective. Another line \cite{frazier2014incentivizing, chen2018incentivizing, Wang2018MultiarmedBW} considered the setting where the arm-pulling history is publicly available to both system and users and the system offers compensations to an arm for incentivized exploration. Our setting follows this line of research.  
%There are also other variants such as

Incentivized learning with monetary payments was first studied in \cite{frazier2014incentivizing} in a Bayesian setting with discounted regret and compensation. %The authors proposed a Bayesian incentivizing model with discounted regret and compensation, and characterized the relationship between the reward, compensation, and discount factor.
\citet{chen2018incentivizing} studied a heterogeneous users setting, where user diversity led to their solution with constant compensation. \citet{agrawal2020incentivising}  considered heterogeneous contexts in a contextual bandit setting. In \cite{Wang2018MultiarmedBW}, the authors analyzed the non-Bayesian and non-discounted reward case and showed $O(\log T)$ regret and compensation in a stochastic MAB setting.   \citet{liu2020incentivized} considered the reward feedback is biased because of the compensation.% and showed that the incentivized exploration with biased reward can still obtain sublinear regret and compensation.
\citet{kannan2017fairness} considered incentivized exploration for fair recommendation. Our setting is mostly similar to \citet{Wang2018MultiarmedBW}, i.e., non-Bayesian and non-discounted reward, but is studied under the linear contextual bandit setting. 
We should note all the aforementioned studies assume the system and the users share the same information such as arm pulls, rewards and contexts, and the system calculates the compensation based on the shared information. Our setting is strictly more challenging. The information gap is caused by information asymmetry: the system cannot access the feature vectors employed by the users. As a result, users' reward estimation will be different from the system' and the precise amount of payment is harder to compute. 

%Multi-armed bandits have been well studied \cite{lai1985asymptotically, auer2002finite} to solve the exploration-exploitation trade-off in a single party game, and its development can be found in several textbooks~\cite{bubeck2012regret,slivkins2017incentivizing, lattimore2020bandit}.We study the problem in the linear contextual bandits setting, which is an important variants of MAB with solutions like UCB-based methods~\cite{li2010contextual,Improved_Algorithm} and Thompson Sampling-based methods~\cite{agrawal2013thompson,abeille2017linear}.  
There are several recent works study low-rank bandits, which however are intrinsically different from ours.
For example, \citet{lale2019stochastic} consider the contexts are sampled from a low-dimensional subspace and propose a PCA-based solution to reduce the dimension. \citet{yang2021impact} study multi-task linear bandits with a shared low-rank structure. These methods assume the learning problem is generated from a low-rank structure but presented in a high-dimensional space. But in our setting, the system's observed contexts are already sampled from a high-dimensional compact space, whose dimension cannot be further reduced. %For example, when the system can only observe the indices of the arms, the best context representation it can use is the one-hot vector and its dimension cannot be further compressed.
The information gap in representation asymmetry is a unique problem in this two-party game setting.

\section{Conclusions and Future Work}
In this paper, we introduced a new and practically-motivated problem of incentivized exploration under information gap in linear contextual bandits. The key challenge is the information asymmetry in the observed context features between a system and a myopic user. We proposed an algorithm that offers sufficient compensation to guarantee users to follow LinUCB's exploration strategy. We proved the regret and compensation upper bound of our algorithm are in the order of  $O(d_v\sqrt{T}\log T)$ under information gap and $O(d_x\sqrt{T}\log T)$ without information gap. We also analyzed the compensation lower bound of the problem. As our future work, we plan to study how to incentivize the users following other types of exploration strategy such as Thompson Sampling~\cite{chapelle2011empirical,agrawal2013thompson,abeille2017linear}. It is also important to investigate whether we can obtain a gap-independent $\Omega(\sqrt{T})$  compensation lower bound to match with the upper bound.

\bibliography{main}
\bibliographystyle{icml2021}
\clearpage
\appendix
\onecolumn
\section{Proof Details}
% \subsection{}

\begin{proof}[Proof of Lemma~\ref{lemma:cb}]
According to the definition of confidence interval, $\textit{CB}_{v,t}(\bv_a) = \alpha_{v,t}\lVert\bv_a\rVert_{\bA_{v,t}^{-1}}$ and $\textit{CB}_{x,t}(\bx_a) = \alpha_{x,t}\lVert\bx_a\rVert_{\bA_{x,t}^{-1}}$. %We prove $\lVert\bv_a\rVert_{\bA_{v,t}^{-1}} \geq \lVert\bx_a\rVert_{\bA_{x,t}^{-1}}$ and $\alpha^v_{t} \geq \alpha^x_{t}$ separately.
We first prove that $\lVert\bv_a\rVert_{\bA_{v,t}^{-1}} \geq \lVert\bx_a\rVert_{\bA_{x,t}^{-1}}$.
By Eq~\eqref{eq:assumption}, we have $\bA_{x,t} - \lambda\bI = \sum_{i=1}^t \bx_{a_i}\bx_{a_i}^\mt = \sum_{i=1}^t P\bv_{a_i}\bv_{a_i}^\mt P^\mt = P(\bA_{v,t}- \lambda\bI)P^\mt$ and
\begin{align*}
\lVert\bx_a\rVert_{\bA_{x,t}^{-1}} &= \sqrt{\bx_a^\mt \bA_{x,t}^{-1} \bx_a}  \\
&= \sqrt{\bv_a^\mt P^\mt \left(\left(P(\bA_{v,t}- \lambda\bI)P^\mt\right)+\lambda\bI\right)^{-1} P \bv_a}
\end{align*}
We prove 
\begin{equation*}
\bv_a^\mt \bA_{v,t}^{-1} \bv_a \geq  \bx_a^\mt \bA_{x,t}^{-1} \bx_a =  \bv_a^\mt P^\mt \left(\left(P(\bA_{v,t}- \lambda\bI)P^\mt\right)+\lambda\bI\right)^{-1} P \bv_a    
\end{equation*}

by showing $ \bA_{v,t}^{-1} -  P^\mt \left(\left(P(\bA_{v,t}- \lambda\bI)P^\mt\right)+\lambda\bI\right)^{-1} P$ is a positive semi-definite matrix based on the property of Schur complement.%, which we stated below. 
% \begin{lemma}

% \end{lemma}

Denote
\begin{equation*}
M=
\begin{bmatrix}
\bA_{v,t}^{-1} & P^\mt\\
P & \left(P(\bA_{v,t}- \lambda\bI)P^\mt\right)+\lambda\bI
\end{bmatrix} .
\end{equation*}
%and $M$ is a symmetric matrix
We have
\begin{align*}
M/\bA_{v,t}^{-1} &= \left(P(\bA_{v,t}- \lambda\bI)P^\mt\right)+\lambda\bI - \left(P^\mt\right)^\mt {\bA_{v,t}} P^\mt \\
&= P\bA_{v,t}P^\mt - \lambda P P^\mt +\lambda\bI - P{\bA_{v,t}} P^\mt \\
&= \lambda\left(\bI - PP^\mt\right) \\
&\succeq 0
\end{align*}
The last step holds because $P$'s largest singular value is smaller than 1, the eigenvalues of $PP^T$ are smaller than 1 and $\bI - PP^\mt \succeq 0$. 
% \huazheng{Do we need to state the property of Schur complement as a lemma?}
Because $\bA_{v,t}^{-1} \succ 0$ and $M/\bA_{v,t}^{-1} \succeq 0 $, according to the property of Schur complement we have $M \succeq 0$. Because $\left(P(\bA_{v,t}- \lambda\bI)P^\mt\right)+\lambda\bI = \bA_{x,t}\succ 0$ and $M \succeq 0$, applying the property again we have $M/\left(\left(P(\bA_{v,t}- \lambda\bI)P^\mt\right)+\lambda\bI\right) \succeq 0$, which gives us $\bA_{v,t}^{-1} - P^\mt \left(\left(P(\bA_{v,t}- \lambda\bI)P^\mt\right)+\lambda\bI\right)^{-1} P \succeq 0$. By the definition of positive semi-definite matrix, we have $\bv_a^\mt \bA_{v,t}^{-1} \bv_a -  \bv_a^\mt P^\mt \left(\left(P(\bA_{v,t}- \lambda\bI)P^\mt\right)+\lambda\bI\right)^{-1} P \bv_a \geq 0$, which means $\lVert\bv_a\rVert_{\bA_{v,t}^{-1}} \geq \lVert\bx_a\rVert_{\bA_{x,t}^{-1}}$.

According to Lemma~\ref{lemma:define_cb}, $\alpha^v_t = R\sqrt{d_v \log \frac{1+t/\lambda}{\delta}} +\sqrt{\lambda}$ and $\alpha^x_t = R\sqrt{d_x \log \frac{1+t/\lambda}{\delta}} +\sqrt{\lambda}$. Since $d_v \geq d_x$, we have $\alpha^v_{t} \geq \alpha^x_{t}$.
Combining the two results and we finished the proof of $\textit{CB}_{v,t}(\bv_a) \geq \textit{CB}_{x,t}(\bx_a)$ holds for any arm $a$ at any time $t$. 
\end{proof}

\begin{proof}[Proof of Theorem~\ref{theorem:comp_no_gap}]
Following the definition of total compensation,  we have
\begin{align*}
\text{C}(T) &=\sum_{t=1}^T \bbE[c_{a_t, t}]\\
&= \sum_{t=1}^T \left(\max_i \hat r_{x,i,t} - \hat r_{x,a_t,t}\right)\\
&\leq \sum_{t=1}^T \left(\max_i \left(\hat r_{x,i,t} + \textit{CB}_{x,t}(\bx_{i})\right) - \hat r_{x,a_t,t}\right)\\
&= \sum_{t=1}^T \left(\hat r_{x,a_t,t} + \textit{CB}_{x,t}(\bx_{a_t}) - \hat r_{x,a_t,t}\right)\\
&= \sum_{t=1}^T \textit{CB}_{x,t}(\bx_{a_t})\\
\end{align*}
where the third step holds with probability at least $1-\delta$ and the fourth step is based on the UCB arm selection strategy. 

So with probability at least $1-\delta$, we bound the total compensation as follows,
\begin{align*}
\text{C}(T) &\leq \sum_{t=1}^T \textit{CB}_{x,t}(\bx_{a_t})\\
&\leq \sqrt{T \sum_{t=1}^T \textit{CB}^2_{x,t}(\bx_{a_t})}\\
&= \sqrt{T \sum_{t=1}^T \alpha_{x,t}^2 \lVert\bx_a\rVert_{\bA_{x,t}^{-1}}^2}\\
&\leq \sqrt{T \alpha_{x,T}^2\sum_{t=1}^T \lVert\bx_a\rVert_{\bA_{x,t}^{-1}}^2}\\
&\leq  \alpha_{x,T}\sqrt{T\sum_{t=1}^T \lVert\bx_a\rVert_{\bA_{x,t}^{-1}}^2}\\
\end{align*}
According to Lemma 11 of~\cite{Improved_Algorithm}, $\sum_{t=1}^T \lVert\bx_a\rVert_{{\bA_{x,t}}^{-1}}^2 \leq d_x\log(\lambda + T/d_v)$. Combining with $\alpha_{x,t} = R\sqrt{d_x \log \frac{1+t/\lambda}{\delta}} +\sqrt{\lambda}$ and we finished the proof.
\end{proof}

\begin{proof}[Proof of Theorem~\ref{theorem:regret}]
We bound cumulative regret by
\begin{align*}
\text{R}(T) &= \sum_{t=1}^T\big(\bbE[r_{a_t^*}] -\bbE[r_{a_t}]\big)\\
&= \sum_{t=1}^T\left(\bv_{a_t^*}^\mt\btheta_v^*-\bv_{a_t}^\mt\btheta_v^*\right)\\
&\leq \sum_{t=1}^T\left(\bv_{a_t^*}^\mt\hat\btheta_{v,t} + 2\textit{CB}_{v,t}(\bv_{a_t^*})-\bv_{a_t}^\mt\btheta_v^*\right)\\
&\leq \sum_{t=1}^T\left(\bv_{a_t}^\mt\hat\btheta_{v,t} + 2\textit{CB}_{v,t}(\bv_{a_t})-\bv_{a_t}^\mt\btheta_v^*\right)\\
&\leq \sum_{t=1}^T 2\textit{CB}_{v,t}(\bv_{a_t})\\
\end{align*}
The third step holds with probability at least $1-\delta$ according to the definition of confidence interval. The fourth step holds with probability at least $1-2\delta$ according to Lemma~\ref{lemma:incentive}, where the users are incentivized to pull arms according to UCB exploration strategy as shown in Eq~\eqref{eq:relaxed_UCB}. Taking a union bound and the above inequality holds with probability at least $1-3\delta$. 

We continue bounding the cumulative regret with probability at least $1-3\delta$ as follows,
\begin{align*}
\text{R}(T) &\leq 2\sqrt{T \sum_{t=1}^T \textit{CB}^2_{v,t}(\bv_{a_t})}\\
&= 2\sqrt{T \sum_{t=1}^T \alpha_{v,t}^2 \lVert\bv_a\rVert_{\bA_{v,t}^{-1}}^2}\\
&\leq  2\alpha_{v,T}\sqrt{T\sum_{t=1}^T \lVert\bv_a\rVert_{\bA_{v,t}^{-1}}^2}\\
&\leq \left(2R\sqrt{d_v \log \frac{1+T/\lambda}{\delta}} +\sqrt{\lambda}\right)\sqrt{Td_v\log(\lambda + \frac{T}{d_v})}
\end{align*}
where we finished the proof by combining $\sum_{t=1}^T \lVert\bv_a\rVert_{\bA_{v,t}^{-1}}^2 \leq d_v\log(\lambda + T/d_v)$ and $\alpha_{v,t} = R\sqrt{d_v \log \frac{1+t/\lambda}{\delta}} +\sqrt{\lambda}$ together.
\end{proof}

\begin{proof}[Proof of Theorem~\ref{theorem:comp}]
With probability at least $1-2\delta$, we have
\begin{align*}
\text{C}(T) &\leq \sum_{t=1}^T 4\textit{CB}_{v,t}(\bv_{a_t})\\
&\leq 4\sqrt{T \sum_{t=1}^T \textit{CB}^2_{v,t}(\bv_{a_t})}\\
&= 4\sqrt{T \sum_{t=1}^T \alpha_{v,t}^2 \lVert\bv_a\rVert_{\bA_{v,t}^{-1}}^2}\\
&\leq  4\alpha_{v,T}\sqrt{T\sum_{t=1}^T \lVert\bv_a\rVert_{\bA_{v,t}^{-1}}^2}\\
&\leq \left(4R\sqrt{d_v \log \frac{1+T/\lambda}{\delta}} +\sqrt{\lambda}\right)\sqrt{Td_v\log(\lambda + \frac{T}{d_v})}
\end{align*}

\end{proof}

\begin{proof}[Proof of Theorem~\ref{theorem:comp_lb_no_gap}]

We first prove that after a fixed time point, with high probability pulling arm $a$ once requires compensation at least $\Delta_{a}/3$. The proof idea is similar to the proof of Theorem 1 in \cite{Wang2018MultiarmedBW}. %but using concentration bound in Lemma~\ref{lemma:concentration} and. 
We then derive the asymptotic compensation lower bound.% using Lemma~\ref{lemma:minimum_exploration}.

Based on Lemma~\ref{lemma:minimum_exploration}, we can obtain
\begin{equation}\label{eq:minimum_exploration}
    \limsup_{T \to \infty}{\log(T) \Vert \bx_a \Vert^2_{G_{x,T}^{-1}} \leq \frac{\Delta_{a}^2}{2}} 
\end{equation}
which is also stated in the Corollary 2 in \cite{lattimore2017end}.

Let $N_a(T)$ be the  number of times arm $a$ is pulled in $T$ rounds.
Since the algorithm has $o(T)$ regret, we can find $T_1'(\delta)$ such that the best arm is pulled at least $T/2$ times with probability $1-\delta/2$. Using the concentration bound we know there exists $T_1''(\delta)$ such that for $t > T_1''(\delta)$ with probability $1-\delta/2$ the confidence interval of the best arm's reward estimation is smaller than $\Delta_{2}/3$ %, i.e.,$\hat r_{x,1,t} \geq \bE[r_1] - \Delta_2/3 $ 
where $\Delta_{2}$ is the reward gap between the best arm and second best arm. Let $T_1(\delta) = \max(T_1'(\delta), T_1''(\delta))$ and for all $t>T_1(\delta)$, with probability $1-\delta$ we have $\hat r_{x,1,t} \geq \bE[r_1] - \Delta_2/3 $.
% \begin{align*}
%     &P[\exists t>T_1(\delta), \hat r_{x,1,t} \leq \bE[r_1] - \Delta_2/3]\\
%     \leq & \sum_{n>N_1^*}P[N_1(t)=n, \hat r_{x,1,t} \leq \bE[r_1] - \Delta_2/3] \\
%     \leq & \delta
% \end{align*}

We argue a similar result for any suboptimal arm $a$. Based on Eq~\eqref{eq:minimum_exploration}, there exists a $T_a(\delta)$ such that for any $t > T_a(\delta)$, with probability $1-\delta$
\begin{equation*}
    \Vert\bx_a\Vert^2_{G_{x,t}^{-1}} \leq \frac{\Delta_{a}^2}{2\log(T)} \leq \frac{\Delta_{a}^2}{9f_{T,\delta}}
\end{equation*}
Combining with the concentration bound in Lemma~\ref{lemma:concentration} and we have for any $t > T_a(\delta)$ with probability $1-\delta$, $\hat r_{x,a,t} - \bE[r_a] \leq \Delta_a/3 $. 

Let $T(\delta) = \max_i T_i(\delta)$ and we know that for any $t>T(\delta)$, the minimum required compensation to incentivize the user to pull arm $a$ is
\begin{equation}
    \max_i\hat r_{x,i, t} - \hat r_{x, a, t} \geq
    \hat r_{x, 1, t} - \hat r_{x, a, t} \geq \bE[r_1] - \frac{\Delta_{2}}{3} - \bE[r_a] - \frac{\Delta_{a}}{3} \geq \frac{\Delta_{a}}{3}
\end{equation}
 with probability at least $1-\delta$.
 
We then use the optimization problem in Eq~\eqref{eq:optimization} to obtain the compensation lower bound, where the optimization minimizes the total compensation and satisfies the consistent constraints that the gaps of all suboptimal arms are identified with high confidence. With probability at least $1-\delta$, for sufficiently large $T$ the total compensation is
\begin{align*}
    C(T)\geq \sum_{a\in \mathcal{A}} \bE[N_a(T)]\frac{\Delta_a}{3}
\end{align*}
$\alpha_{\bx_a} = \bE[N_a(T)]/\log(T)$ is asymptotically feasible for large $T$ because it satisfies
\[
    \limsup_{T \to \infty}{ \Vert \bx_a \Vert^2_{H_{x,T}^{-1}}}=\limsup_{T \to \infty}{\log(T) \Vert \bx_a \Vert^2_{G_{x,T}^{-1}} \leq \frac{\Delta_{a}^2}{2}} \]
where $G_{x,T} = \log(T)H_{x,T}$.
Thus for any $\epsilon>0$,  $\Vert \bx_a \Vert^2_{H_{x,T}^{-1}} \leq \Delta_{a}^2/2 + \epsilon$ and
\begin{align}
    C(T)\geq \sum_{a\in \mathcal{A}} \bE[N_a(T)]\frac{\Delta_a}{3} \geq c_{x,\epsilon}(\mathcal{A}, \btheta^*)\log(T)
\end{align}
where $c_{x,\epsilon}(\mathcal{A}, \btheta^*)$ is the the optimal value of the optimization problem in Eq~\eqref{eq:optimization} by replacing $\Delta_{a}^2/2$ with $\Delta_{a}^2/2 + \epsilon$. Since $\inf_{\epsilon>0}c_{x,\epsilon}(\mathcal{A}, \btheta^*) = c_{x}(\mathcal{A}, \btheta^*)$ and $T\to\infty$ we have the total compensation as
\begin{equation*}
    \Omega\left(c_x(\mathcal{A}, \btheta^*)\log(T)\right)
\end{equation*}

\end{proof}

\end{document}